\documentclass{article}

    \usepackage[preprint,nonatbib]{neurips_2022}

\usepackage[utf8]{inputenc} 
\usepackage[T1]{fontenc}    
\usepackage{hyperref}       
\usepackage{url}            
\usepackage{booktabs}       
\usepackage{amsfonts}       
\usepackage{nicefrac}       
\usepackage{microtype}      
\usepackage[dvipsnames,table]{xcolor}

\usepackage{caption}
\usepackage{multirow} 
\usepackage{amsmath,eucal,amssymb,mathrsfs,amsthm}
\usepackage{booktabs}
\usepackage{enumerate}
\usepackage{placeins}
\usepackage{graphicx}
\usepackage{longtable}
\usepackage{threeparttable}
\usepackage{float}
\usepackage{math_command}
\usepackage{thm}

\definecolor{citecol}{HTML}{6F130C}
\definecolor{tableofcontent}{HTML}{1F4A83}
\definecolor{urlcol}{HTML}{2470D8}

\usepackage{hyperref}
\hypersetup{
    colorlinks=true,       
    linkcolor=tableofcontent,          
    citecolor=citecol,        
    urlcolor=blue,           
}

\title{
Lower and Upper Bounds for Numbers of Linear Regions of Graph Convolutional Networks
}

\author{%
  Hao Chen \\
  Wuhan University\\
  Shanghai Jiao Tong University\\
  \texttt{ch2111009@gmail.com} 
   \And
   Yu Guang Wang \\
   Shanghai Jiao Tong University\\
   University of New South Wales\\
   \texttt{yuguang.wang@sjtu.edu.cn} 
  \And
  Huan Xiong\\
  Harbin Institute of Technology\\
  Mohamed bin Zayed University of AI\\
  \texttt{huan.xiong.math@gmail.com}
}

\begin{document}

\maketitle

\begin{abstract}
The research for characterizing GNN expressiveness attracts much attention as graph neural networks achieve a champion in the last five years. The number of linear regions has been considered a good measure for the expressivity of neural networks with piecewise linear activation. In this paper, we present some estimates for the number of linear regions of the classic graph convolutional networks (GCNs) with one layer and multiple-layer scenarios. In particular, we obtain an optimal upper bound for the maximum number of linear regions for one-layer GCNs, and the upper and lower bounds for multi-layer GCNs. The simulated estimate shows that the true maximum number of linear regions is possibly closer to our estimated lower bound. These results imply that the number of linear regions of multi-layer GCNs is exponentially greater than one-layer GCNs per parameter in general. This suggests that deeper GCNs have more expressivity than shallow GCNs.
\end{abstract}

\section{Introduction}\label{sec:intro}
Graph Neural Networks (GNNs) \cite{scarselli2008graph} have gained a champion for various structured prediction tasks. GNNs have the similar network architecture as the traditional CNNs \cite{goodfellow2016deep} but taking the graph structured data as input. Compared with CNNs, GNNs take account of the feature on the node, like a pixel feature, and also geometric property of the data in the network propagation. In particular, the feature of each node is extracted from layer to layer by aggregating its neighbor information \cite{gilmer2017neural}. While a good many GNN models have been proposed, the expressiveness of GNNs is yet to be characterized in theory. The number of linear regions has been a useful tool for capturing how complex a neural network is, and shown quantitatively why deep neural networks surpasses shallow networks \cite{Mont2014On,Xiong2020OnTN}. The linear regions can also reveal the expressiveness of GNNs, especially when the GNNs take the form of graph convolutional networks (GCNs) \cite{kipf2016semi}. 
In this paper, we prove the lower and upper bound of the number of linear regions of GCNs and show how the depth links to the linear regions. 

Different kinds of deep neural networks (DNNs) have shown outstanding performance in some circumstances such as sound recognition, computer vision, and recommendation system \cite{hinton2012deep,goodfellow2013maxout,abdel2014convolutional,silver2016mastering,vaswani2017attention,he2020lightgcn,zheng2021disentangling}.
People have developed theories for the powerful expressivity of DNNs to explain the outstanding performance of different neural networks. 
For example, a one-layer fully connected neural network is proved to have the power to approximate any continuous function with enough width \cite{Hornik1991ApproximationCO,barron1994approximation}. Any given Lebesgue-integrable function can be approximated by a neural network with enough layers \cite{hanin_boris,lu2017expressive}. Yarosky has provided the approximation error of deep neural networks \cite{yarotsky2017error}.
Other works quantify the impact of the network structure on expressivity for deep neural networks. 
Telgarsky \cite{telgarsky2015representation} has proved that the depth helps reduce the number of parameters: a shallow network needs much more parameters than a deep neural network to represent a same function.

\begin{figure}[t]
\centering   
 \begin{minipage}{\textwidth}
 \centering
\begin{minipage}{0.25\textwidth}
	\centering         
	\includegraphics[width=\textwidth]{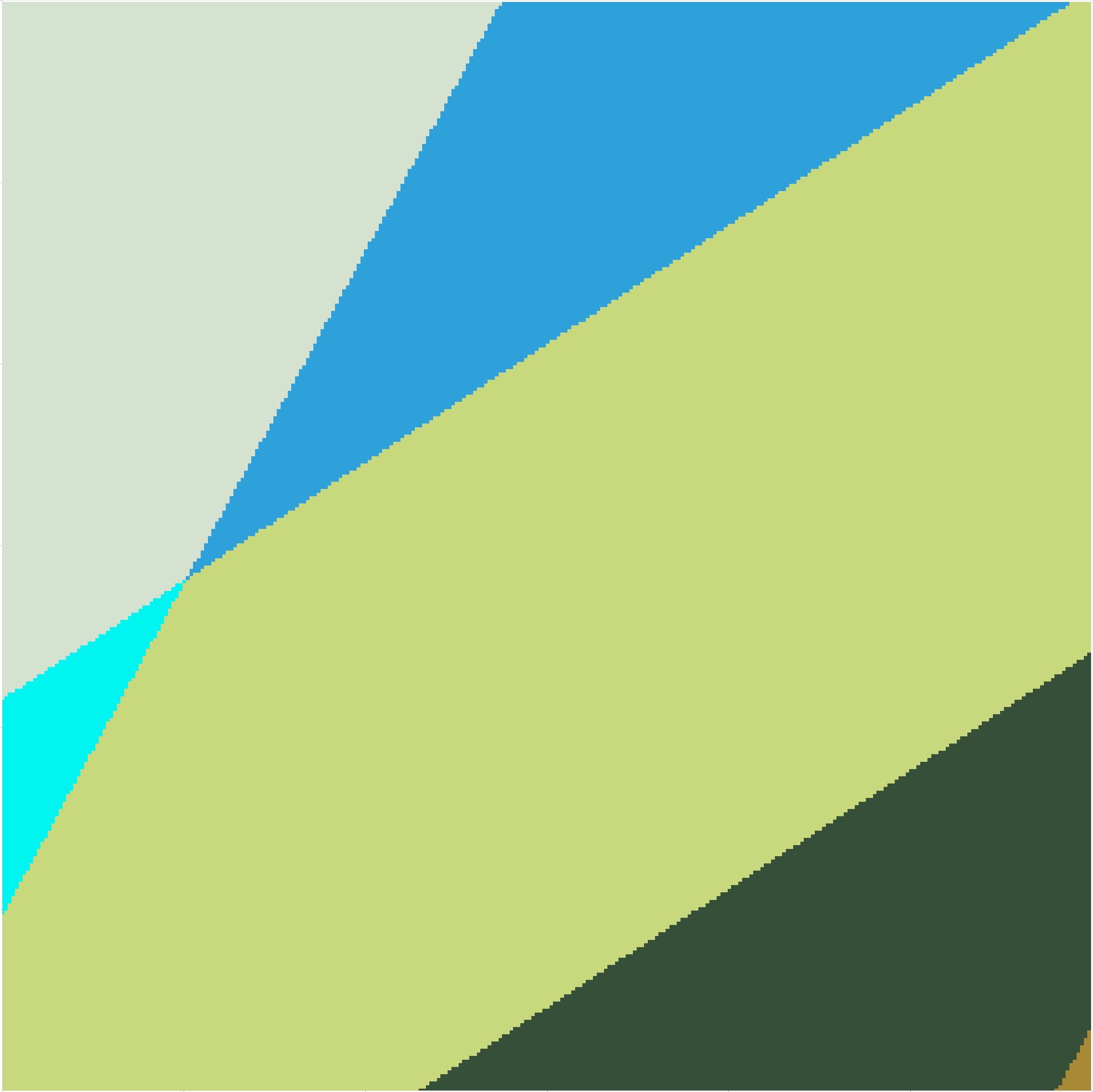}   
	\end{minipage}
	\hspace{3mm}
\begin{minipage}{0.25\textwidth}
	\centering         
	\includegraphics[width=\textwidth]{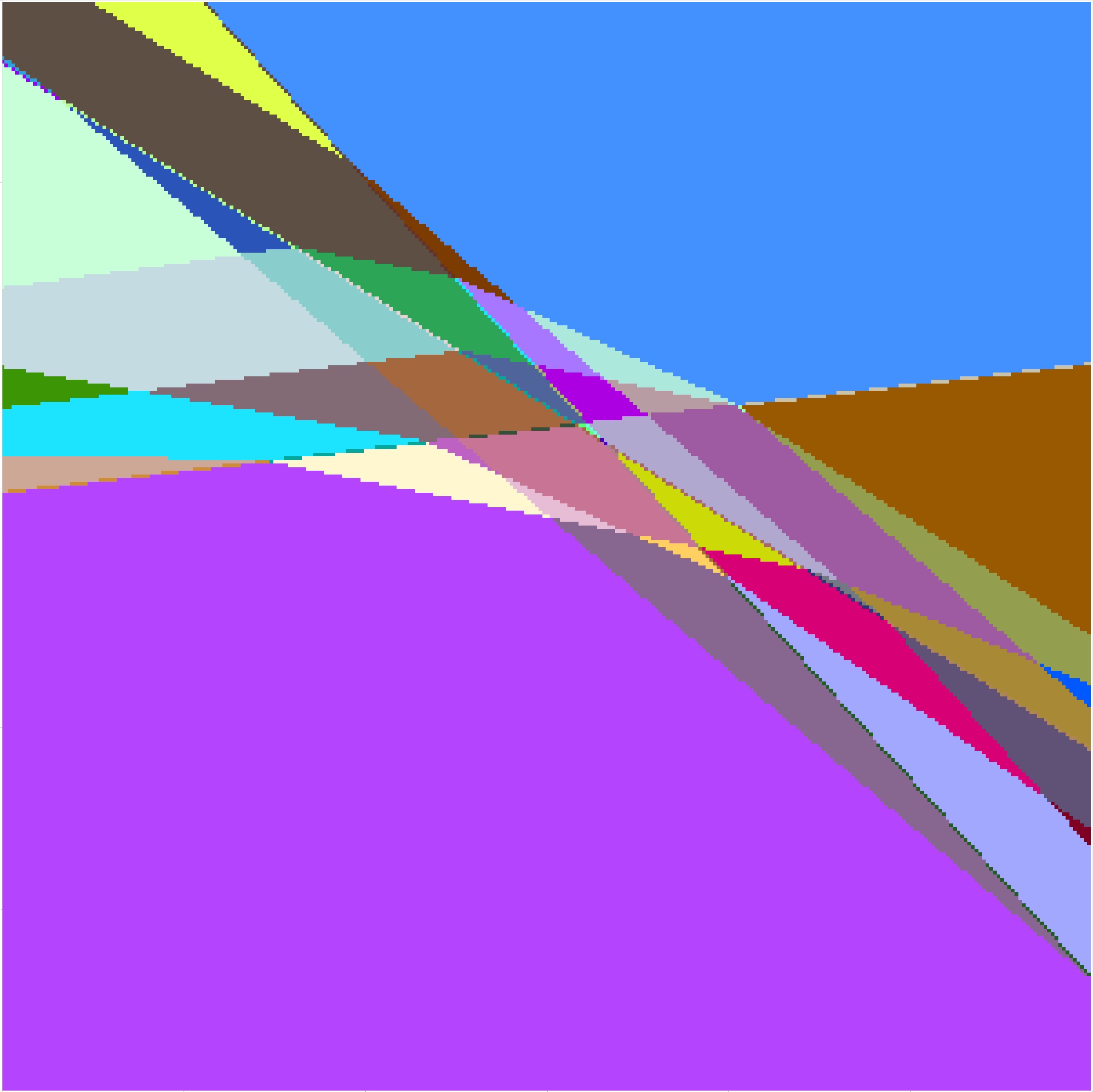}   
	\end{minipage}
	\hspace{3mm}
\begin{minipage}{0.25\textwidth}
	\centering     
	\includegraphics[width=\textwidth]{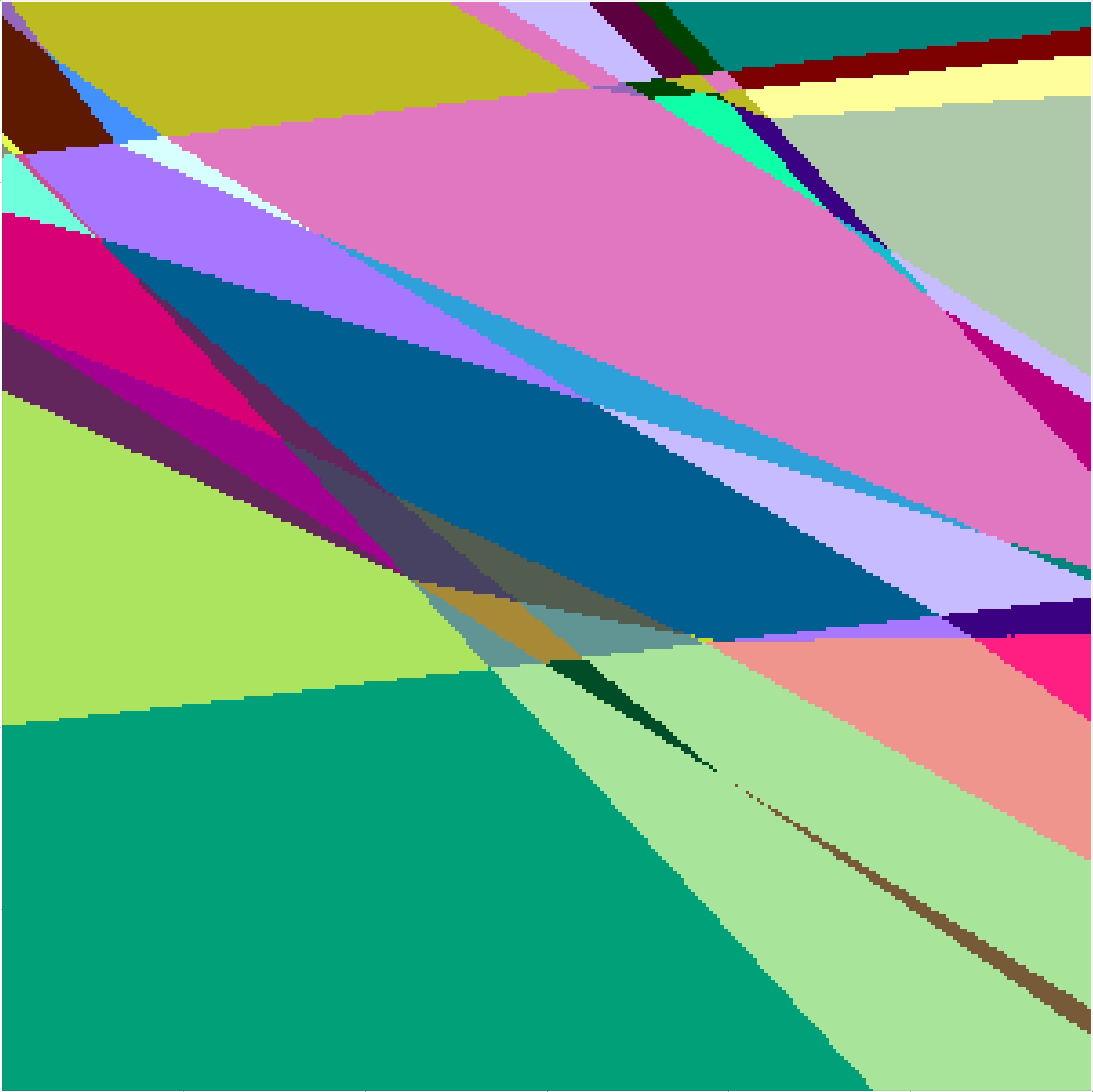}
	\end{minipage}
\end{minipage}
\caption{2D slices of linear regions for the input space of 3 GCNs with 1 layer (left), 2 layers (middle) and 3 layers (right), on an input graph with 3 nodes and 2 edges. Each GCN has 1 input feature for the first layer and 3 output features for all layers.}
\label{fig:linear-regions}
\end{figure}

A variety of methods have been proposed to characterize the ability of a network to approximate any given function, such as number of linear regions \cite{pascanu2013number,Mont2014On,chen2021neural}, length distortion \cite{price2021trajectory,hanin2021deep}, VC-dimension \cite{vapnik2015uniform}, and other topological and geometric tools \cite{h.2018on,liang2019fisher}. In this paper, we focus on the number of linear regions for GCNs. To count linear regions of neural networks, researchers mainly studied the NNs with Rectifier Linear Unit (ReLU) activation, a piecewise linear function introduced by \cite{hahnloser2000digital,hahnloser2000permitted}, which has been widely used in various kinds of neural networks \cite{glorot2011deep}. The feedforward propagation of a neural network activated with ReLU from one hidden layer to the next can be regarded as a piecewise linear function. Then, the whole neural network can be viewed as a composition of several piecewise linear functions, which is also piecewise. Neural networks with more linear pieces can estimate more complex functions. GCNs that are neural networks for structured data input can be also viewed as piecewise functions, and the number of the linear regions shows the expressive power of GCNs.

\paragraph{Activation Patterns of Neural Networks} Activation pattern of a neural network with piecewise linear activation can be characterized by the number of activation regions. It is important to measure the network complexity, approximation error, convergence rate of networks, implicit bias in parameter optimization and Lipschitz constants. Nguyen et al. (2020) \cite{DBLP:journals/corr/abs-2012-11654} showed that linear regions can be applied in the research of the eigenvalues of the neural tangent kernel and Lipschitz constant. Linear regions have been a useful tool for capturing how complex a neural network is, and shown quantitatively why deep neural networks surpass shallow networks \cite{pascanu2013number,Mont2014On,Xiong2020OnTN}. Steinwart et al. (2019) \cite{DBLP:journals/corr/abs-1903-11482} reported that the convergence of gradient descent is influenced by activation patterns. Phuong et al. (2020) \cite{phuong2020functional} linked linear regions with function-preserving transformations of network parameters.

\paragraph{Maximum and Expected Numbers of Linear Regions} The maximum number of linear regions is one of the most useful measures for activation patterns. Mont\'{u}far et al. (2014) \cite{Mont2014On} reported the upper and lower bounds for the maximum number of linear regions of NNs. The upper bound was later improved by Serra et al. (2018) \cite{serra2018bounding}. Based on \cite{hinz2019framework}, a tighter upper bound was proved by Hinz et al. (2021) \cite{Hinz2021UsingAH}.
As the activate pattern is largely affected by the network architecture and number of parameters, estimating the expected number of linear regions with respect to the choice of parameters is significant for comparing the network complexity among different architectures. The expected number of linear regions for ReLU networks was first reported in \cite{hanin2019complexity,hanin2019deep}. Tseran et al. (2021) showed the expected number of linear regions for maxout networks \cite{tseran2021expected}.

\paragraph{Related Notion} The activation regions can be measured from different perspectives, such as tropical geometry, max-affine splines, and trajectory length. Zhang et al. \cite{zhang2018tropical} showed that the family of ReLU neural networks is equivalent to the family of tropical rational maps. The trajectory length was first proposed in \cite{raghu2017expressive} to measure how the network output changes as the input sweeps along a one-dimensional path. To study the piecewise affine convex nonlinearities, the new operator was proposed to concatenate the max-affine spline formed from affine mappings \cite{balestriero2019geometry}. 

\paragraph{Contributions} Using specific topological and geometrical methods, this work obtains the lower and upper bounds for the number of one-layer and multi-layer GCNs. Below are our main contributions.
\begin{itemize}
    \item We obtain an optimal upper bound for the number of linear regions for one-layer GCNs when the network parameters are in a complement of a measure zero set with respect to space $\mathbb{R}^{\#{\rm weights}+\#{\rm biases}}$.
    \item We obtain the upper and lower bounds for the maximal number of linear regions of multi-layer GCNs. In our theory, we translate the problem of counting areas in the input space into a problem of counting hyper-cubes in the new space. 
    The lower bound is proved by constructing a specific GCN. 
    \item We show that the number of linear regions of multi-layer GCNs is exponentially greater than one-layer GCNs, which suggests that deeper GCNs have more expressivity than their shallow counterparts. 
    \item We also compare the number of linear regions for GCNs, NNs, and CNNs with the same input dimension and similar network architectures. It shows that the multi-layer ReLU (activated) GCNs have more linear regions than multi-layer ReLU CNNs and multi-layer ReLU NNs. 
\end{itemize}


\section{Preliminary}\label{sec:pre}
In this section, we give some basic notations and definitions. Let $\mathbb{N}$, $\mathbb{N}^+$ and $\mathbb{R}$ be the set of nonnegative integers, positive integers and real numbers respectively. Let $\#S$ be the number of elements in a set $S$. In this paper, we focus on GCNs with $L$-layer convolutions and exclude pooling layers and fully connected layers for simplicity. Let $G=(V,E)$ denote the graph of the GCNs where $V$ and $E$ are the set of vertices and edges of the graph. Let $D:=\#V$ be the number of nodes, and $N_0$ be the number of input features (i.e. feature dimension). Therefore, the input dimension of the first layer is $D\times N_0$. Suppose that there are $N_l$ output features in the $l$-th layer. Let $A$ be the adjacency matrix of graph $G$. We define $\hat{A}=M^{-1/2} (I+A)M^{-1/2}$ with $M$ being the degree matrix of $I+A$. Let $\tilde{A}$ be the $\hat{A}$ with repeated rows deleted, and $D^*=\rank(\hat{A})$ which is also the number of rows of $\tilde{A}$.

The Rectifier Linear Unit (ReLU) is adopted as the activation function for each neuron in every hidden layer. Let $X^{(0)}\in \mathbb{R}^{D\times N_0}$ be the input of the first layer and $X^{(l)}\in \mathbb{R}^{D\times N_l}$ be the output of the $l$-th layer. Let $W_l\in \mathbb{R}^{N_{l-1}\times N_l}$ and $b_l\in \mathbb{R}^N_l$ be the weight and bias of $l$-th layer. Therefore, we denote the weights and biases for the GCNs as $W=(W_1,W_2,\dots,W_L)$ and $B=(b_1,b_2,\dots,b_L)$. Denote by $\theta=(W,B)$ trainable network parameters including weights and biases. 

As we can represent a one-layer GCN with given weights and biases by a composition of an affine linear function and a piece-wise linear function, we can see the GCN as a piece-wise linear function. Thus, for any given $W$ and $B$, the $L$-layer GCN can be seen as a composition of $L$ piece-wise linear functions, which we then write as a piece-wise linear function $\mathcal{F}_{\mathcal{N},W,B}(X^{(0)})$.

\begin{definition}[Activation Patterns and Linear Regions] \label{def:ac-pa}
 Let $\mathcal{N}$ be an $L$-layer GCN with $n$ neurons in total. An activation pattern is a function of all inputs (denoted as $z_i,\ i=1,2,\dots,n$) of an activation function of $n$ neurons in $\mathcal{N}$. The value of an activation pattern is in $\{-1,1\}^n$. Suppose that $P_i$ is the $i$-th component of $P(z_1,z_2,\dots,z_n)$. A region corresponding to $P$ and parameters $\theta$ reads 
$$
\mathcal{R}(P,\theta)=\{X^{(0)}:\forall i\ \in\ [1,n],\ P_i\cdot z_i>  0\}.
$$
A non-empty $\mathcal{R}(P,\theta)$ is a linear region for $\mathcal{N}$ when the parameter is $\theta$. Therefore, the number of linear regions of $\mathcal{N}$ with parameter $\theta$ is $R_{\mathcal{N},\theta}=\#\{\mathcal{R}(P,\theta):\exists P\ such\ that\ \mathcal{R}(P,\theta)\neq \varnothing\}$. Furthermore, denote the maximal number of linear regions by $R_{\mathcal{N}}$.
\end{definition}

\begin{remark}
From Definition \ref{def:ac-pa}, $\mathcal{F}_{\mathcal{N},W,B}$ is a linear affine function when restricted in one linear region. Then, the $L$-layer GCN can represent a piece-wise linear function with $R_{\mathcal{N},\theta}$ pieces. Therefore, we can count the number of linear regions to measure the expressive power of a GCN. A formula or bounds of the number of linear regions will be shown for single- and multi-layer GCNs in our work.
\end{remark}

The definition of activation pattern implies that the maximal number of activation patterns for $\mathcal{N}$ is $2^n$,
as stated in the next lemma.
\begin{lemma}\label{lemma:trivial-bound}
A trivial upper bound of the number of linear regions is 
\begin{equation}
    R_{\mathcal{N},\theta}\leq 2^n.
\end{equation}
\end{lemma}

To state the main results, we introduce the asymptotic notation: For two functions $f(x)$ and $g(x)$, we write $f(x)=\mathcal{O}(g(x))$ if $\exists\  c>0$ s.t. $f(x)\leq cg(x)$ for all $x$ greater than some constant; $f(x)=\Omega(g(x))$ if $\exists\  c> 0$ s.t. $f(x)\geq cg(x)$ for all $x$ greater than some constant; $f(x)=\Theta(g(x))$ if $\exists\  c_1>0$ and $c_2> 0$ s.t. $c_1g(x)\leq f(x)\leq c_2g(x)$ for all $x$ greater than some constant. We put all proofs of the results of the paper in the appendix.

\section{Number of Linear Regions for One-Layer GCNs}
In this section, we show the exact formula for the maximal number and average number of linear regions of GCNs. The estimate shows that if the parameters of GCNs are drawn from some distribution, the average number of linear regions is exactly the maximal number of linear regions. Therefore, it is almost sure that the number of linear regions of GCNs with arbitrary given parameters can reach the maximal number. 

\paragraph{Exact Formulas for One-Layer GCNs}
The next proposition about linear regions of fully connected NNs is needed to derive the exact formula for One-Layer GCNs.
\begin{proposition}\label{prop:nn-eq}
Let $\mathcal{N}$ be a one-layer fully connected neural network. Assume that the input dimension of $\mathcal{N}$ is $n_0$ and there are $n_1$ neurons. Then, the maximal number of linear regions of $\mathcal{N}$ is $\sum_{i=0}^{n_0}\tbinom{n_1}{i}$.
\end{proposition}
Since $\mathcal{N}$ is a fully connected neural network, the pre-activation of a hidden layer neuron is the affine function of the input, which then translates the problem of counting the number of linear regions to counting the number of regions of hyperplane arrangements in the general position, and this then implies the following formula for the number of linear regions of one-layer GCNs.
\begin{theorem}\label{thm:one-layer-ineq}
Assume that $\mathcal{N}$ is a one-layer GCN with the input $X^{(0)}\in \mathbb{R}^{D\times N}$ and the output $X^{(1)}\in \mathbb{R}^{D\times N^{\prime}}$. Let $\Tilde{A}$ denote normalized adjacency matrix $A$ with duplicated rows removed and $D^*=\text{rank}(A)$. Let $V_{1,1}, V_{1,2}, \dots, V_{D^*N^{\prime}}$ be the subspaces of $V=\mathbb{R}^{D^*\times N}$ and $V_{ij}=\alpha_i^{\top}w_j^{\top}$ where $\alpha_i$ is the $i$-th column of $\hat{A}$ and $w_j$ is the $j$-th row of $W$. For given network parameters $\theta$, let 
\begin{equation}\label{eq:K_N}
    K_{\mathcal{N},\theta}:=\left\{(k_{1,1},k_{1,2},...,k_{D^*N^{\prime}}):k_{ij}\in \mathbb{N},\sum_{i\in I,\: j\in J} k_{ij} \leq {\rm dim}\bigl(\sum_{k_{ij}=1}V_{ij}\bigr),\; \forall (i,j)\in [D^*]\times [N^{\prime}]\right\}.
\end{equation}
Then, we have the upper bound 
\begin{equation}\label{eq:rn-kn}
    R_{\mathcal{N},\theta}\leq \sum\limits_{(k_{1,1},...,k_{D^*N^{\prime}})\in K_{\mathcal{N},\theta}}\;\prod\limits_{(i,j)\in [D^*]\times [N^{\prime}]}^{D^*}\binom{1}{k_{ij}}=\#K_{\mathcal{N},\theta}.
\end{equation}
\end{theorem}

To prove Theorem~\ref{thm:one-layer-ineq}, we creatively change the form of adjacency matrix $\hat{A}$ by deleting the repeated rows and constructing several arrangements not in general position with the row and column number. From the generalization form of Zaslavsky theorem in \cite{Xiong2020OnTN}, we derive \eqref{eq:rn-kn}.

From the above Theorem~\ref{thm:one-layer-ineq}, the rank of $A$ and $W$ is closely related to the number of elements in $K_{\mathcal{N},\theta}$ in \eqref{eq:K_N}. We have the next proposition which shows the upper bound for the maximal number of linear regions in two scenarios for whether $N>N^{\prime}$ or $N\leq N^{\prime}$.
\begin{proposition}\label{col:max-one-layer}
Let $\rank(\Tilde{A})=D^*$. The next two estimates for $K_{\mathcal{N},\theta}$ in \eqref{eq:K_N} hold.

(\romannumeral1) When $N >N^{\prime}$ and $\mathrm{rank}(W)=N^{\prime}$,
\begin{equation}\label{eq:kn1}
    \#K_{\mathcal{N},\theta}=2^{N^{\prime}D^*}.
\end{equation}

(\romannumeral2) When $N \leq N^{\prime}$ and any N column vectors in W are linearly independent,
\begin{equation}\label{eq:kn2}
    \#K_{\mathcal{N},\theta}=\left(\sum_{i=0}^{N}\tbinom{N^{\prime}}{i}\right)^{D^*}.
\end{equation}
\end{proposition}

Theorem~\ref{thm:one-layer-ineq} and Proposition~\ref{col:max-one-layer} show that in general case $R_{\mathcal{N},\theta}$ is upper bounded by $\#K_{\mathcal{N},\theta}$. The next theorem shows that the probability of $R_{\mathcal{N},\theta}=\#K_{\mathcal{N},\theta}$ is $1$.
\begin{theorem}\label{thm:one-layer-eq}
Assume that the parameters $W$ and $B$ are drawn from some distribution $\mu$ which has densities with respect to Lebesgue measure in $\mathbb{R}^{\#weights+\#biases}$. Then, the number of linear regions of the GCN
$R_{\mathcal{N},\theta}=\left(\sum_{i=0}^{N}\tbinom{N^{\prime}}{i}\right)^{D^*}$ almost surely, that is, when network parameters are in the complement of a measure-zero set of $\mu$ on $\mathbb{R}^{\#{\rm weights}+\#{\rm biases}}$. 
Moreover, the expectation is
\begin{equation}\label{eq:expected R_N}
    E_{\theta\sim \mu}(R_{\mathcal{N},\theta})=\left(\sum_{i=0}^{N}\tbinom{N^{\prime}}{i}\right)^{D^*}.
\end{equation}
\end{theorem}

Theorem~\ref{thm:one-layer-eq} relies on the following lemmas.
\begin{lemma}\label{lemma:adj-matrix-rank}
Let the adjacency matrix of graph G for one GCN be $\Tilde{A}$. If each row of $\Tilde{A}$ is different, $\Tilde{A}$ is invertible. If we take away the rows that have already appeared in $\Tilde{A}$, then the left row vectors are linearly independent.  
\end{lemma}

\begin{lemma}\label{lemma:W-rank2}
All $W\in \mathbb{R}^{N\times N^{\prime}}$ satisfying $\mathrm{rank}(W)\textless \min(N,N^{\prime})$ form a measure zero set in $\mathbb{R}^{N\times N^{\prime}}$.
\end{lemma}

\begin{lemma}\label{lemma:W-rank}
For $W\in \mathbb{R}^{N\times N^{\prime}}$ and $N\leq N^{\prime}$, all $W$ such that $\exists\ \ I\subseteq [N^{\prime}]$ and $\#I=N$ satisfying ${\rm rank}((W_i,\ i\in I))\textless N$ form a measure-zero set in $\mathbb{R}^{N\times N^{\prime}}$.
\end{lemma}

\begin{lemma}
\label{lemma1}
Let $W\in \mathbb{R}^{N\times N^{\prime}}$ and $B\in \mathbb{R}^{N^{\prime}}$. All pairs  $(W,B)$ satisfying that there exists one vector $p_w$ such that $p_w {\rm vec}(B)=0$, form a null set in $\mathbb{R}^{D\times N^{\prime}+N^{\prime}}$ with respect to Lebesgue measure. 
\end{lemma}

\textbf{Outline for the proof of Theorem~\ref{thm:one-layer-eq}.} With the results in Theorem~\ref{thm:one-layer-ineq} and Proposition~\ref{col:max-one-layer}, we need only to show that all $W$ and $B$ satisfying $R_{\mathcal{N},\theta}\textless\left(\sum_{i=0}^{N}\tbinom{N^{\prime}}{i}\right)^{D^*}$ form a measure zero set for $\mu$ on $\mathbb{R}^{\#weights+\#biases}$. By Lemma~\ref{lemma:adj-matrix-rank}, we can show that the rank of the adjacency matrix is equal to the number of distinct rows. The operation of deleting the repeated rows in the adjacency matrix helps us to obtain Lemma~\ref{lemma1}, which implies that $R_{\mathcal{N},\theta}=\#K_{\mathcal{N},\theta}$ almost surely. By Lemmas~\ref{lemma:W-rank} and \ref{lemma:W-rank2}, \eqref{eq:kn1} and \eqref{eq:kn2} then hold almost surely on $\mathbb{R}^{\#weights+\#biases}$. 

Denote $R_{\mathcal{N}}$ as the optimal upper bound of $R_{\mathcal{N},\theta}$ that is $\left(\sum_{i=0}^{N}\tbinom{N^{\prime}}{i}\right)^{D^*}$.
Table~\ref{tab:example1} compares the estimate for a one-layer GCN with input dimension $3\times 1$ and output dimension $3\times N_1$ by  Theorem~\ref{thm:one-layer-ineq} and Proposition~\ref{col:max-one-layer} with Proposition~\ref{prop:nn-eq} and Lemma~\ref{lemma:trivial-bound}, where $N_1$ ranges in $\{1,2,3,4,5\}$. 
From \eqref{eq:expected R_N}, $R_{\mathcal{N}}=(1+N_1)^3$ almost surely. The upper bound for $R_{\mathcal{N}}$ obtained from Lemma~\ref{lemma:trivial-bound} is $2^{3d}$, and by Proposition~\ref{prop:nn-eq}, $\sum_{i=0}^{3}\tbinom{3d}{i}$. The table results show that Theorem~\ref{thm:one-layer-ineq} which gives the tight upper bound provides a precise estimate than the upper bound of Proposition~\ref{prop:nn-eq} and the trivial bound of Lemma~\ref{lemma:trivial-bound}.
We compare the estimates as Table~\ref{tab:example1} for another two types of graphs in Section~\ref{append:examples}.

\begin{table}[th]
\centering
\caption{Comparison of estimates for one-layer GCN with input, output size $3\times 1$, $3\times N_1$.}
\label{tab:example1}\vspace{2mm}
\begin{tabular}{c|ccccc}
\toprule
 & $N_1=1$ & $N_1=2$  & $N_1=3$   & $N_1=4$    & $N_1=5$     \\ \midrule
$R_{\mathcal{N}}$ (Thm~\ref{thm:one-layer-eq}, Cor~\ref{col:max-one-layer}) & 8 & 27 & 64  & 125  & 216   \\ 
 Upper bound (Prop~\ref{prop:nn-eq}) & 8 & 42 & 130 & 299  & 576   \\ 
 Naive upper bound (Lem~\ref{lemma:trivial-bound}) & 8 & 64 & 512 & 4096 & 32768 \\ \bottomrule
\end{tabular}
\end{table}

\paragraph{Asymptotic Estimate}

Let $R_{\mathcal{N}}$ be the maximum of the $R_{\mathcal{N},\theta}$ over all possible parameters $\theta$. The following theorem shows an asymptotic estimate for $R_{\mathcal{N}}$.

\begin{theorem}\label{thm:asym}
Let $\mathcal{N}$ be a one-layer GCN. Suppose that $N^{\prime}$ tends to infinity and other parameters are fixed integers. Let $\rank(\hat{A})=D^*$, then 
\begin{equation}
    R_{\mathcal{N}}=\Theta\bigl((N^{\prime})^{D^* N}\bigr).
\end{equation}
\end{theorem}

\section{Bounds of the Number of Linear Regions for Multi-Layer GCNs}
The following theorem presents upper and lower bounds for the maximal number of linear regions for multi-layer GCNs. Both lower and upper bounds can be written by the maximum number of linear regions for a one-layer GCN multiplied by an exact number depending on the network depth and the number of neurons. The theory applies to general network architectures of GCNs.
\begin{theorem}
\label{thm:bounds}
Let $\mathcal{N}$ be a GCN with L hidden layers and ReLU activation. Suppose that the input is $X^{(0)}\in \mathbb{R}^{D\times N_0}$ and the output of each layer is $X^{(l)}\in \mathbb{R}^{D\times N_l}$. Let the adjacency matrix be $\hat{A}\in \mathbb{R}^{D\times D}$ and the weight of each layer be $W\in \mathbb{R}^{N_0\times N_l}$. Assume that $N_l\geq N_0$ for each $l\geq 1$. Then, we have the following lower and upper bounds of $R_{\mathcal{N}}$.

(i) The maximal number of linear regions $R_{\mathcal{N}}$ of GCN $\mathcal{N}$ is at most:
\begin{equation}\label{ineq:up-multi}
    R_{\mathcal{N}}\leq R_{\mathcal{N^{''}}}\prod\limits_{l=2}^{L}\left(\sum\limits_{i=0}^{DN} \tbinom{DN_l}{i}\right),
\end{equation}
where $\mathcal{N}''$ is a one-layer GCN with input and output dimensions $D\times N_0$ and $D\times N_1$.

(ii) The maximal number of linear regions $R_{\mathcal{N}}$ of GCN $\mathcal{N}$ is no less than:
\begin{equation}\label{ineq:low-multi}
      R_{\mathcal{N}}\geq R_{\mathcal{N^{\prime}}}\prod\limits_{l=1}^{L-1}\left \lfloor \frac{N_l}{N} \right \rfloor^{N\times \rank(A)},
\end{equation}
where $\mathcal{N}'$ is a one-layer GCN with input and output dimensions $D\times N$ and $D\times N_{L}$.

\end{theorem}

We derive \eqref{ineq:up-multi} by induction. Assume that \eqref{ineq:up-multi} is true for $L-1$ layers. Then the input space could be separated into several linear regions by the first $L-1$ layers. In each linear region, the network consists of the first $L-1$ layers and could be viewed as a linear function of the input. Then, in each linear region, the network consisting of $L$ layers could be viewed as a fully connected neural network whose upper bound is given in \cite{Mont2014On}.

To obtain the lower bound in Theorem~\ref{thm:bounds}, we design a GCN with specific parameters. Both the adjacency matrix and weight have an influence on the input features $X$ so we can use the space folding method introduced in \cite{Mont2014On} directly. We consider the adjacency matrix as a linear mapping from the input space $V_X$ into another space $V_Y$. With specific parameters, we could separate an area in $V_Y$ into several hypercubes and each could be mapped to the original area in the input space. The eigenvalue and eigenvector properties in Lemma~\ref{lemma:eigen} make such transformation possible. Therefore, with Lemma~2 in \cite{Mont2014On}, we then obtain the lower bound.

\section{Expressivity Comparison of Different Network Architectures}

We compare the number of parameters in deep and shallow graph convolutional networks. First, for an $L$-layer GCN, its number of parameters can be counted exactly. This demonstrates the complexity of deep GCNs compared to shallow GCNs and MLPs with the same-order number of parameters.
\begin{lemma}\label{lem:5.1}
Suppose that there are D nodes in the graph, $N_0$ input features and $N_l$ output features for the $l$-th layer. Then the number of network parameters is $\sum_{l=1}^{L}(N_{l-1}\times N_{l}+N_{l})$.
\end{lemma}
\subsection{Deep GCNs vs Shallow GCNs}
By Lemma \ref{lem:5.1}, Theorem \ref{thm:one-layer-ineq} and Theorem \ref{thm:bounds}, it is easy to obtain the following results.
\begin{theorem}
\label{thm4.1}
Let $\mathcal{N}_1$ be an L-layer GCN. Suppose that $N_1=N_2=\cdots =N_L=N^{\prime}$ tends to infinity. Then, $\mathcal{N}_1$ has $L(NN^{\prime}+DN^{\prime})=\Theta(LN^{\prime})$ parameters and the ratio $\rho_{\mathcal{N}_1}$ of $R_{\mathcal{N}_1}$ to the number of parameters of $\mathcal{N}_1$ is bounded by
\begin{equation}
    \rho_{\mathcal{N}_1}=\Omega\left(\frac{1}{L}{N^{\prime}}^{LN\times \rank(A)-1}\right).
\end{equation}
On the other hand, let $\mathcal{N}_2$ be a one-layer GCN with input dimension $D\times N$ and output dimension $D\times LN^{\prime}$. Then $\mathcal{N}_2$ also has  $\Theta(LN^{\prime})$ parameters, which has the same-order number of parameters as $\mathcal{N}_1$. The ratio $\rho_{\mathcal{N}_2}$ of $R_{\mathcal{N}_2}$ to the number of parameters of $\mathcal{N}_2$ is
\begin{equation}
    \rho_{\mathcal{N}_2}=\Theta\left((LN^{\prime}D)^{N\times \rank(A)-1}\right).
\end{equation}
Since only $N^{\prime}$ tends to infinity, we obtain that $\rho_{\mathcal{N}_2}=\Theta((N^{\prime})^{N\times \rank(A)})$.
\end{theorem}

It is not difficult to check that $\rho_{\mathcal{N}_1}$ grows much faster when $L$ and $N^{\prime}$ are large enough. Theorem~\ref{thm4.1} then shows that with the same total number of parameters, a deep GCN has more complexity than a shallow GCN, and thus more expressive power.

\subsection{GCNs vs Fully Connected NNs}
Let $\mathcal{N}_1$ be the $L$-layer GCN in Theorem \ref{thm4.1} with invertible adjacency matrix and $\mathcal{N}_3$ be an $L$-layer fully connected NNs (or an $L$-layer MLP). Suppose that the input dimension of NNs is $D\times N_0$ and the output dimension of each layer is $N$. Trivially, the number of parameters for $\mathcal{N}_1$ and $\mathcal{N}_2$ are both $\mathcal{O}(LN^2)$ so we only need to compare the maximum number of linear regions. Again, we compare the upper bound of linear number of $\mathcal{N}_2$ with the lower bound of $\mathcal{N}_1$. Since $R_{\mathcal{N}_1}=\Theta\left((DN_0)^L\tbinom{N}{DN_0}^L\right)$, from Stirling's formula \cite{Flajolet09analyticcombinatorics} we could obtain $R_{\mathcal{N}_1}=\Theta\left((DN_0)^L\frac{N^{LDN_0}}{(DN_0/e)^{DN_0L}\sqrt{2\pi DN_0}^L}\right)$. Therefore, 
\begin{equation}\label{eq:R_N1 to R_N2}
    \frac{R_{\mathcal{N}_1}}{R_{\mathcal{N}_2}}=\Omega\left(\left(\frac{D^{DN_0}}{DN_0e^{DN_0}}\right)^L\right).
\end{equation}
Obviously, when D tends to infinity, the ratio $\frac{R_{\mathcal{N}_1}}{R_{\mathcal{N}_2}}$ approaches infinity. This shows that $R_{\mathcal{N}_1}$ grows much faster than $R_{\mathcal{N}_2}$. This observation means ReLU GCNs with invertible adjacency matrix has more expressivity than NNs. 

\begin{remark}
Let $\mathcal{N}_1$ be the $L$-layer GCN in Theorem \ref{thm4.1} with invertible adjacency matrix and $\mathcal{N}_3$ be an $L$-layer CNN. Suppose that the input dimension of the first layer is $D\times 1\times N$ and the output dimension of each layer is $1\times 1\times N^{\prime}$. Therefore, the upper bound of the number of linear regions for $\mathcal{N}_3$ is $R_{\mathcal{N}_3^*}\prod_{l=2}^{L}\sum_{i=0}^{Nd}\tbinom{N^{\prime}}{i}$ where $R_{\mathcal{N}_3^*}$ is the number of linear regions for a one-layer CNN with input dimension $D\times 1\times N$ and output dimension $1\times 1\times N^{\prime}$. From \cite{Xiong2020OnTN} we can prove  $R_{\mathcal{N}_3^*}=\mathcal{O}(R_{\mathcal{N}_1^*})$. Thus, we obtain the same order to \eqref{eq:R_N1 to R_N2}, $\frac{R_{\mathcal{N}_1}}{R_{\mathcal{N}_3}}=\Omega\left(\left(\frac{D^{DN_0}}{DN_0e^{DN_0}}\right)^L\right)$. This means that ReLU GCNs with invertible adjacency matrix has more expressive power than CNNs. Note here we provide a rough comparison as we used the upper bound of CNNs in \cite{Xiong2020OnTN}, which has the same order as the fully connected NNs with same depth and similar size of input and same order number of network parameters. The true maximum number of linear regions of CNNs is smaller than $\mathcal{O}(R_{\mathcal{N}_1^*})$.
\end{remark}

\section{Experiments}\label{append:examples}

\begin{table}[t]
\centering
\caption{The lower and upper bounds from \eqref{ineq:up-multi} and \eqref{ineq:low-multi}, and simulated estimate of Linear regions of GCN with two layers, for different feature dimensions $N_2$ of the second layer output. The input feature dimension for the first and second is both set 2.}
\label{tab:ran-multi}
\vspace{2mm}
\begin{tabular}{c|ccccc}
\toprule
$N_2$                   & 1    & 2   & 3     & 4  &5        \\ \midrule
Lower bound          & 8   & 64  & 343  & 1331 & 4096  \\ 
Simulated Estimate & 101    & 223 & 1643  &  2398&4453      \\ 
Upper bound          & 512 & 4096 & 29824 & 160640 & 636736\\ \bottomrule
\end{tabular}
\end{table}

\paragraph{Number of Linear Regions for 2-layer GCNs} 
We show an example of two-layer GCN to illustrate  lower and upper bounds of Theorem~\ref{thm:bounds} compared with the simulated estimate.
Let $\mathcal{N}$ be a two-layer GCN. 
The inputs of the first layer and the second layer are $X^{(0)}\in \mathbb{R}^{3\times 2}$ and  $X^{(1)}\in \mathbb{R}^{3\times N_1}$, respectively. The output dimension of the second layer is $3\times N_2$. 
The input graph has the normalized adjacency matrix
\begin{equation*}
 \hat{A}=\left (
 \begin{matrix}
   \frac{1}{2} & \frac{1}{\sqrt{6}} & 0 \\
    \frac{1}{\sqrt{6}} & \frac{1}{3} &  \frac{1}{\sqrt{6}} \\
    0 &  \frac{1}{\sqrt{6}} & \frac{1}{2}
  \end{matrix}
  \right).
\end{equation*}

Table~\ref{tab:ran-multi} shows the bounds of different GCNs with different values of $N_2$ together with the simulated estimate for $R_{\mathcal{N}}$ by random sampling methods. The results suggest that the lower bound is closer to the simulated estimate, and thus the true number of the linear regions. 

To validate our results in the previous sections, we have randomly sampled data points as the input, and counted the corresponding number of linear regions by Definition~\ref{def:ac-pa}. We initialize the parameters based on Kaiming He initialization \cite{he2015delving}. Given parameters, we have tried normal distribution and uniform distribution to sample our input points with different distribution parameters and report the biggest number of linear regions in Table \ref{tab:ran-multi}. Concretely, we test for the normal distribution with $\mu=0$ and $\sigma^2\in \{1,3,5,7,9\}$, and for uniform distribution $\mathcal{U}(-u,u)$ with $u\in \{1,5,10\}$. To obtain a statistically sound estimate, we sample $2\times 10^6$ input data and record the activation pattern for each input, and find the total disjoint patterns as the total linear regions for all inputs. This number is summarized in Table~\ref{tab:ran-multi} as the simulated estimate for the linear regions of a specific GCN.

\begin{figure}[t]
\centering   
 \begin{minipage}{\textwidth}
 \centering
\begin{minipage}{0.48\textwidth}
	\centering         
	\includegraphics[width=\textwidth]{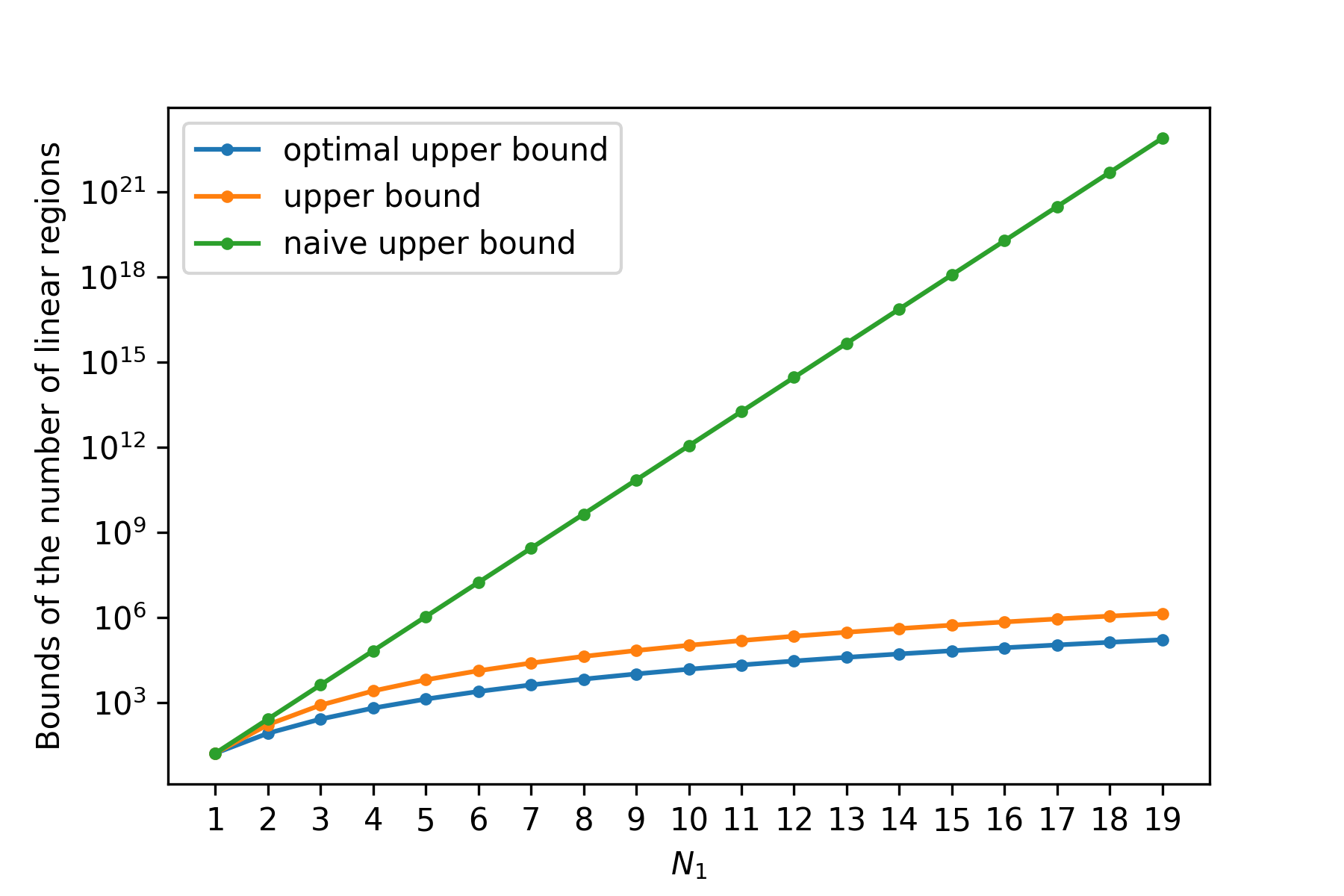}   
	\end{minipage}
\begin{minipage}{0.48\textwidth}
	\centering         
	\includegraphics[width=\textwidth]{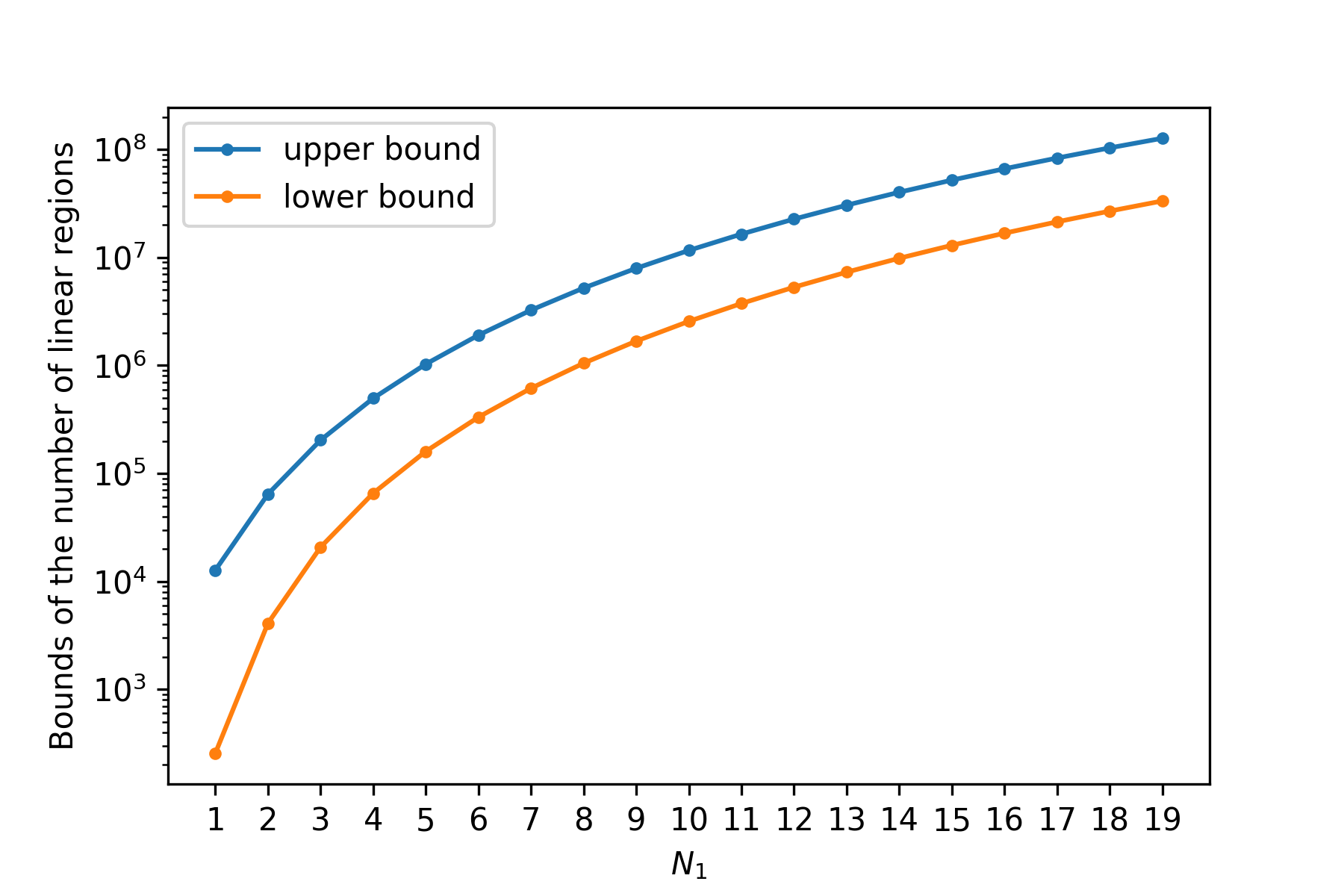}   
	\end{minipage}
\end{minipage}
\caption{Bounds for GCNs with 1 layer (left) and 2 layers (right) on an input graph which consists of 4 nodes and 3 edges. The left image shows the optimal upper bound (green) gained from Theorem~\ref{thm:one-layer-eq}, upper bound (orange) from Proposition~\ref{prop:nn-eq} and the naive upper bound (blue) from Lemma~\ref{lemma:trivial-bound} for a 1-layer GCN with 1 input feature and $N_1$ output features where $N_1$ ranges from 1 to 19. The right image shows the lower bound (orange) and upper bound (blue) obtained from Theorem~\ref{thm:bounds} for a 2-layer GCN with 1 input feature and $N_1$ and 3 output features for the first and second layers.}
\label{fig:linear-region4}
\end{figure}


\paragraph{Estimates for Linear Regions of GCNs with 1 and 2 layers} 
Figure~\ref{fig:linear-region} shows the bounds for the number of linear regions of one-layer and two-layer GCNs as given by Theorem~\ref{thm:one-layer-eq}, Proposition~\ref{prop:nn-eq}, Lemma~\ref{lemma:trivial-bound} and Theorem~\ref{thm:bounds}. The input graph has 4 nodes and adjacency matrix
$$
 \hat{A}=\left (
 \begin{matrix}
   \frac{1}{3} & \frac{1}{3} & \frac{1}{\sqrt{6}} & 0 \\
  \frac{1}{3} & \frac{1}{3} & 0& \frac{1}{\sqrt{6}}  \\
    \frac{1}{\sqrt{6}} &  0 & \frac{1}{2} &0\\
    0&\frac{1}{\sqrt{6}} &0&\frac{1}{2}\\
  \end{matrix}
  \right).
$$
In this specific case, we can obtain exact formulas for the estimated bounds, as follows.
For 1-layer GCN with 1 input feature and $N_1$ output features, from \eqref{eq:expected R_N}, $R_{\mathcal{N}}=(1+N_1)^4$ almost surely. The upper bounds in Proposition~\ref{prop:nn-eq} and Lemma~\ref{lemma:trivial-bound} read $R_{\mathcal{N}}\leq \sum_{i=0}^{4}\tbinom{4N_1}{i}$ and $R_{\mathcal{N}}\leq 2^{4N_1}$. For the 2-layer GCN with 1 input feature and $N_1$ output features for the first layer and 3 output features for the second layer, the upper and lower bounds are given by $625N_1^4\leq R_{\mathcal{N}}\leq (1+N_1)^4\sum_{i=0}^{4N_1}\tbinom{4N_1}{i}$. We use these bounds to plot Figure~\ref{fig:linear-region4}.
The results demonstrate that for one-layer GCNs, our estimated upper bound (orange) for the maximum number of linear regions has an improvement for the existing upper bound (blue).
For both the one-layer and two-layer GCNs, the network becomes more complex as the number of neurons of the first layer increases.


\section{Related Work}\label{sec:related work}
\paragraph{Linear Regions} The number of linear regions for one-layer NNs with $n_0$ inputs and $n_1$ neurons was first reported by \cite{pascanu2013number}. They turned the problem of counting linear regions into the problem of counting regions of hyperplane arrangements in general position and applied Zaslavsky's theorem \cite{zaslavsky1975facing,stanley2004introduction}. Moreover, they proposed a lower bound $\left(\prod_{l=1}^{L-1}\lfloor\frac{n_l}{n_0}\rfloor\right)\sum_{i=0}^{n_0}\tbinom{n_{L}}{n_0}$ for the maximal number of linear regions for a fully connected neural network with input dimension $n_0$ and $L$ hidden layers of width $n_1$, $n_2$,\dots, $n_L$. Based on these results, they concluded that deep fully-connected ReLU NNs have exponentially more maximal linear regions than the shallow counterparts with asymptotically the same number of parameters. 
The exact number of linear regions for one-layer CNNs with input dimension $n_0^{(1)}\times n_0^{(2)}\times d_0$ and output dimension $n_1^{(1)}\times n_1^{(2)}\times d_1$ was first proposed\cite{Xiong2020OnTN}. In their work, they translate the problem of counting the linear regions of CNNs to a problem on counting the regions of some class of hyperplane arrangements which usually are not in general position, and develop a generalized Zaslavsky's theorem to solve this problem. Furthermore, they proposed a lower bound $R_{\mathcal{N}'}\prod_{l=1}^{L-1}\lfloor\frac{d_l}{d_0}\rfloor^{n_l^{(1)}\times n_l^{(2)}\times d_0}$ ($R_{\mathcal{N}}$ is the number of linear regions for the one-layer CNN with input dimension $n_{L-1}^{(1)}\times n_{L-1}^{(2)}\times d_{L-1}$ and output dimension $n_L^{(1)}\times n_L^{(2)}\times d_L$. ) for the maximal number of linear regions for a fully connected neural network with input dimension $n_0^{(1)}\times n_0^{(2)}\times d_0$ and $L$ hidden layers size $n_1^{(1)}\times n_1^{(2)}\times d_1$, $n_2^{(1)}\times n_2^{(2)}\times d_2$,\dots, $n_L^{(1)}\times n_L^{(2)}\times d_L$. Based on these results, they concluded that deep fully-connected ReLU CNNs have exponentially more maximal linear regions than their shallow counterparts with asymptotically the same number of parameters. Besides, they also compared the expressive power of deep ReLU CNNs with deep fully-connected ReLU NNs on some assumptions of the network architecture and found that there are more linear regions in the input space of deep ReLU CNNs than the deep ReLU NNs. 

The exact number of linear regions for GNNs is reported in 2021 \cite{bodnar2021weisfeiler}. They focus on one-layer GNNs without bias and translate the problem of counting the linear regions of CNNs to a problem on counting the regions of some class of hyperplane arrangements in general position. They also obtained that the regions of an arrangement $\mathcal{A}$ of $M$ hyperplanes in $\mathbb{R}^{N}$ is equal to $2\sum_{j=0}^{N-1}\tbinom{M-1}{j}$. Compared to \cite{bodnar2021weisfeiler}, we consider a general case of bias for GCNs, and prove the lower and upper bounds of the maximal number of linear regions of multilayer GCNs.

\paragraph{WL Tests for GNN Expressiveness} \cite{geerts2022expressiveness} gave a thorough study of the expressiveness and approximation properties of GNNs in terms of the separation power and WL-test. For the expressivity of GNNs, WL-tests are designed. The Weisfeiler-Lehman algorithm \cite{weisfeiler1968reduction} was first proposed to solve the graph isomorphism problem. The properties of these WL algorithms are also investigated \cite{babai1980random,cai1992optimal}. Later, $k$-WL and $k$-FWL algorithms were proposed and proved to be more powerful or efficient. For instance, $\forall k\geq 2$, $k$-FWL is as powerful as $(k+1)$-WL and $(k+1)$-WL is strictly more powerful than $k$-WL \cite{grohe2015pebble,grohe2017descriptive}.  \cite{xu2018powerful,morris2019weisfeiler} have proved that if there exist two graphs $G$ and $H$ that are possibly isomorphic from the output of $1$-WL algorithm, the embeddings computed by GNNs must be the same. For any two graphs $G$ and $H$ that are non-isomorphic according to the $k$-WL algorithm, there would exist a $k$-GNN computing different embeddings $h_G$ and $h_H$ \cite{morris2019weisfeiler}. For higher order graph neural networks, $k$-WL algorithm could also be a standard criterion for the expressive power of GNNs and more neural networks on simplicial complexes \cite{maron2019provably,chen2020can,bodnar2021weisfeiler,bodnar2021cwn}. These works could describe the expressive power of a GNN by its ability of distinguishing different graphs.

\section{Discussion}
We focus on the number of linear regions of GCNs with non-zero biases. We give the exact formula for the maximal and average number of linear regions for one-layer GCNs and the upper and lower bounds for multi-layer GCNs. By these results, we compare the number of linear regions per parameter between deep GCNs and their shallow counterparts and among deep GCNs,  CNNs and  fully connected NNs. It turns out that deep GCNs have more expressivity than shallow GCNs.

We finally propose several potential directions related to the current work. 1) From table \ref{tab:ran-multi}, we can find that the stochastic sampling result is close to the lower bound but far from the upper bound. Thus, we seek for tighter bounds or exact formulas. 
2) We obtain the formula for the expectation of the number of linear regions for one-layer ReLU GCNs, which is the same as the maximal number. When the number of layers for a fully connected NN is greater than one, it has been proved in \cite{hanin2019deep} that the expectation of the number of linear regions is usually much smaller than the maximal number. It is interesting to derive similar formulas for the expectation of
the number of linear regions for multi-layer GCNs.  
3) We only consider the number of linear regions for GCNs which is one typical type of GNNs. However, there have been proposed many kinds of GNN models with various graph feature extraction modules including convolutions and message passing. To extend the results of the paper to other graph convolutions or neural message passing is one direction to pursue. 

\bibliographystyle{unsrt}


\newpage
\appendix
\section{Proof of the Number of Linear Regions for One-Layer GCNs}

\subsection{Proof of Theorem \ref{thm:one-layer-ineq}}

To prove Theorem \ref{thm:one-layer-ineq}, we need the following result  from \cite{Xiong2020OnTN}. 

\begin{lemma}[\cite{Xiong2020OnTN}] \label{lemma:xiong}
Let m, n be some positive integers, $V=\mathbb{R}^n$, $V_1$, $V_2$,\dots $V_m$ be m non-empty subspaces of $V$, and $n_1$, $n_2$,\dots $n_m$ be some nonnegative integers. Let $\mathcal{A}=\{H_{kj},\ 1\leq k\leq n,\ 1\leq j\leq n_i\}$ be an arrangement in $\mathbb{R}^n$ with $H_{kj}={X:\alpha X=b_{kj}}$ where $\boldsymbol{0}\neq \alpha_{kj}\in V_k$, $b_{kj}\in \mathbb{R}$. Then, the number of linear regions for the arrangement $\mathcal{A}$ satisfies
\begin{equation}\label{ineq:one-layer}
    r(\mathcal{A})\leq \sum_{(i_1,\ i_2,\ \dots,\ i_m)\in K_{V;V_1,\ V_2,\ \dots,\ V_m}}\prod_{k=1}^{m}\tbinom{n_k}{i_k},
\end{equation}
where 
$$K_{V;V_1, V_2, \dots, V_m}=\left\{(i_1,i_2,\dots,i_m):\ i_k\in \mathbb{N},\ \sum_{k\in J}i_k\leq {\rm dim}(\sum_{i\in J}V_k)\; \forall J\subseteq [m]\right\}.$$

Furthermore, assume that the two conditions hold for the arrangement $\mathcal{A}$:

(i) For each $(i_1,\ i_2,\ \dots,\ i_m)\in K_{V;V_1,\ V_2,\ \dots,\ V_m}$, any $\sum_{k=1}^mi_k$ vectors with $i_k$ distinct vectors chosen from the set $\{\alpha_{kj}:\ 1\leq j\leq n_k\}$ are linearly independent;

(ii) For each $(i_1,i_2,\dots,i_m)\in \mathbb{N}^m \backslash K_{V;V_1,V_2,\dots,V_m}$, the intersection of any $\sum_{k=1}^mi_k$ hyperplanes with $i_k$ distinct hyperplanes chosen from the set $\{H_{kj}:1\leq j\leq n_k\}$ are empty.

Then, the equality in \eqref{ineq:one-layer} holds:
\begin{equation}
    r(\mathcal{A})= \sum_{(i_1,i_2,\dots,i_m)\in K_{V;V_1,V_2,\dots,V_m}}\prod_{k=1}^{m}\tbinom{n_k}{i_k}    .
\end{equation}
\end{lemma}

Now, we could give proofs of Theorem \ref{thm:one-layer-ineq} and Proposition \ref{col:max-one-layer}.
\begin{proof}[Proof of Theorem \ref{thm:one-layer-ineq}]
Assume that $\mathcal{N}$ is a one-layer GCN with input dimension $D\times N$ like this:
\begin{equation}
    X^{(1)} = h(\hat{A}X^{(0)}W),
\end{equation}
where $\hat{A} \in \mathbb{R}^{D\times D}$, $W \in \mathbb{R}^{N\times N^{\prime}}$ and $h$ is the activation function. Let $V_{1,1},V_{1,2},...,V_{D^*N^{\prime}}$ be subspaces of $V=\mathbb{R}^{D^*\times N}$. Suppose $\mathcal{A}=\{H_{ijk},1\leq i \leq D^*, 1\leq j \leq N^{\prime}, 1\leq k\leq q\}$ is an arrangement in $V$ with $H_{ijk}=\{X:\left<\alpha_{ij},X\right>_F+b_{ij}=e_k\}$, where $\alpha_{ij}=a_i^{\top}w_j^{\top} \in V_{ij}$ and $a_i$ is the $i$-th row of $\hat{A}$ and $w_j$ is $j$-th column of $W,$ and $\left<,\right>_F$ is the Frobenius inner product. Then, ${\rm dim}(V_{ij})=1$ and ${\rm dim}(\Sigma_{i\in I,\ j\in J}V_{ij})={\rm rank}((w_j^{\top}),j\in J)\times {\rm rank}((a_i^{\top}),i\in I)$, $\forall J\subseteq [D^*],\ I\subseteq [N^{\prime}]$. We thus obtain
\begin{equation}
    K_{\mathcal{N},\theta}:=\{(k_{1,1},k_{1,2},...,k_{D^*N^{\prime}}):k_{ij}\in \mathbb{N},\sum_{i\in I,\ j\in J} k_{ij} \leq {\rm dim}(\sum_{k_{ij}=1}V_{ij}) , \forall (i,j)\in [D^*]\times [N^{\prime}].
\end{equation}

From the definition of arrangement $\mathcal{A}$, the number of linear regions for $\mathcal{A}$ is equal to the number of linear regions for the input space of $\mathcal{N}$.
The maximal number of linear region of $\mathcal{N}$ is then
\begin{equation}
    R_{\mathcal{N},\theta}\leq\sum\limits_{(k_{1,1},...,k_{DN^{\prime}})\in K_{\mathcal{N}}}\prod\limits_{(i,j)\in [D^*]\times [N^{\prime}]}^{D}\binom{1}{k_{ij}}=\#K_{\mathcal{N}}.
\end{equation}
\end{proof}

\begin{proof}[Proof of Proposition \ref{col:max-one-layer}]
By Lemma \ref{lemma:adj-matrix-rank}, all row vectors in $\Tilde{A}$ are linearly independent. Then, if the distinct vectors $\alpha_{ij}$ are linearly independent, for any chosen index $i$ in $\alpha_{ij}$ (assuming that $J^*_i$ is the index set for all $j$), the vectors $w_j$ in $\alpha_{ij}$ are linearly independent. On the other hand, for any chosen $i$, if the vectors $w_j$ (for all $j\in J^*_i$) are linearly independent, the vectors $\alpha_{ij}$ for all $i$ and $j$ are linearly independent. Otherwise, we could assume that $\exists\ k_{ij}$ that are not all zero could make $\sum_{i=1}\sum_{j\in J^*_i}k_{ij}a^{\top}_iw^{\top}_j=0$. Therefore, $\forall\ i,\ \sum_{j\in J^*_i}k_{ij}w^{\top}_j=0$ and this is contrary to the fact that for all the vectors $w_j$ (for all $j\in J^*_i$) are linearly independent.    

We then obtain if and only if there are a group of $\alpha_{ij}$ linearly independent, the corresponding group of $W_j$ (for all $j$ in $J^*$) are linearly independent. Besides, since $k_{ij}\in \{0,1\}$, the number of the elements of $K_{\mathcal{N},\theta}$ is equivalent to the number of ways of choosing a group of independent vectors from all $\alpha_{ij}$.

(\romannumeral1) 
Since the rank of $W$ is $N^{\prime}$, the column vectors in $W$ are linearly independent. If $\mathrm{rank}(W)\textless N^{\prime}$, not all $\alpha_{ij}$ are linearly independent, which indicates that there will be less elements in $K_{\mathcal{N},\theta}$. We thus obtain that $\#K_{\mathcal{N},\theta}= 2^{D^*N^{\prime}}$.

(\romannumeral2) 
If any $N$ column vectors in $W$ are linearly independent, we could choose at most $N$ column vectors from $W$ so that all the corresponding $\alpha_{ij}$ are linearly independent. If not, there will be less elements in $K_{\mathcal{N},\theta}$. This implies that $$\#K_{\mathcal{N},\theta}= \left(\sum_{i=0}^{N}\tbinom{N^{\prime}}{i}\right)^{D^*}.$$ 
\end{proof}

\subsection{Proof of Theorem \ref{thm:one-layer-eq}}
To prove Theorem~\ref{thm:one-layer-eq}, we need to show that when the rows $\hat{A}$ are linearly independent to each other and $\rank(W)=\min(N,N^{\prime})$, the condition 1 and condition 2 are satisfied, which means the equality in theorem can be achieved. Next, we need to prove that when we take away all the repeated rows of $\hat{A}$, the left rows are linearly independent. Moreover, it is not difficult to check that if $W$ and $B$ are drawn from a distribution which has densities with respect to Lebesgue measure in $\mathbb{R}^{\#weights+\#biases},$ the set of $W$ with $\rank(W)=\min(N,N^{\prime})$ is a complement of a measure-zero set with respect to the Lebesgue measure on $\mathbb{R}^{\#{\rm weights}}$. To achieve these, we need Lemmas~\ref{lemma:adj-matrix-rank} to \ref{lemma1}, which we prove now.
\begin{proof}[Proof of Lemma~\ref{lemma:adj-matrix-rank}]
If we can find one row vector of $\Tilde{A}$ denoted as $\alpha_k$ such that
$$\alpha_k=\sum\limits_{i\in S}a_i\alpha_i,$$ 
where $\alpha_i$ is the $i$-th row of $\Tilde{A}$ and $a_i> 0$. Since the element of $\Tilde{A}$ is either 1 or 0 and $\Tilde{A}_{ii}=1$, we obtain that $\Tilde{A}_{ki}=1$. Moreover, as $\Tilde{A}$ is symmetric, $\Tilde{A}_{ik}=1$. Thus $\sum\limits_{i\in S}a_i=1$. Now we claim that for any $j$ that $\Tilde{A}_{kj}=1$, $\Tilde{A}_{ij}=1\ \forall i\in S$ and that for any $j$ that $\Tilde{A}_{kj}=0$, $\Tilde{A}_{ij}=0$. Therefore, $\alpha_k=\alpha_i\ \forall\ i\in S$, which is contrary to the hypothesis. Then, $\Tilde{A}$ is invertible. We can then check that if we eliminate the duplicate rows of $\Tilde{A}$, the left row vectors are linearly independent.
\end{proof}

\begin{proof}[Proof of Lemma~\ref{lemma:W-rank2}]
Without loss of generality, let $N\geq N^{\prime}$. Suppose that $W_0$ be the first $N^{\prime}$ rows of $W$. 
Consider the set $\mathcal{W}=\{W\in \mathbb{R}^{N\times N^{\prime}}:\  W_0\ \text{is}\ \text{not}\ \text{invertible}\}$. From Lemma~6 in \cite{Xiong2020OnTN}, $\mathcal{W}$ is a Lebesgue measure-zero set in $\mathbb{R}^{N\times N^{\prime}}$. Then, for any given $N$ and $N^{\prime}$, all $W$ such that the matrix consisting of any $N^{\prime}$ rows of $W$ is not invertible form a measure zero set.  
On the other hand, if the matrix formed by any $N^{\prime}$ rows of $W$ is not invertible, $\mathrm{rank}(W)\textless \min(N,N^{\prime})$. Then, all $W$ such that $\mathrm{rank}(W)\textless \min(N,N^{\prime})$ form a null set in $\mathbb{R}^{N\times N^{\prime}}$.
\end{proof}



\begin{proof}[Proof of Lemma~\ref{lemma:W-rank}]
Since there are $N$ column vectors in $W$ that are linearly dependent, there exists a $w_{i_0}$ in these $N$ vectors that could be linearly expressed as the linear sum of the other $N-1$ vectors. Therefore, from Lemma 6 in \cite{Xiong2020OnTN}, all the matrices consisting of these $N$ vectors form a measure zero set in $\mathbb{R}^{N\times N}$. This implies that all $W$ form a null zero set in $\mathbb{R}^{N\times N}$.
\end{proof}


\begin{proof}[Proof of Lemma~\ref{lemma1}]
Denote the set of $(W,B)$ as $E$.
Let the $(X_1,\mathcal{B}_1,\mu_1)$,  $(X_1,\mathcal{B}_1,\mu_1)$ be Lebesgue measure space. Let $X_1=\mathbb{R}^{N\times N^{\prime}}$ and $X_2=\mathbb{R}^{D\times N^{\prime}}$. Then, $W\in X_1$, $B\in X_2$ and $E\in X_1\times X_2$. Let $\mu_1$ and $\mu_2$ be two Lebesgue measures, $\mu=\mu_1\times \mu_2$ and $E^W=\{B:(W,B)\in E\}$.
Then, $\mu(E^W)=0$. We thus obtain $\mu(E)=\int \mu_1(E^W) d\mu_2=0$.
\end{proof}

\begin{proof}[Proof of Theorem \ref{thm:one-layer-eq}]
With Theorem \ref{thm:one-layer-ineq}, we only need to consider the two conditions in Lemma \ref{lemma:xiong}. 

Condition (\romannumeral1): If we choose $\sum_{i,j}a_{ij}$ vectors with $a_{ij}$ vectors from $V_{ij}$, and $\sum_{i,j} a_{ij} \leq {\rm dim}(\sum_{i,j} V_{ij})$, these vectors are linearly independent.

Condition (\romannumeral2): Let $A^*=\{\alpha_{ij}:\; H_{ij}\; \hbox{~is~chosen} \}$, $B^*=\{b_j: H_{ij}\; \hbox{~is~chosen}\}$, $I_H=\{(i,j):\; H_{ij}\; \hbox{~is~chosen}\}$. Then, the intersection of the chosen hyperplanes is empty is equivalent to say that there is no solution to the linear functions $\left\langle\alpha_{ij},X\right\rangle_F=b_j$. The linear functions could be turned into the matrix form $E \cdot {\rm vec}(X)=B$ with $E$ consisting of all the vectorized $\alpha_{ij}$. Due to the way of choosing $H_{ij}$, the rank of $E$ is less than its row number. Therefore, there exists a vector $\rho_E=(r_{1,1},\dots,r_{ij},\dots)$ such that $\sum_{(i,j)\in I_H}r_{ij}\alpha_{ij}=0$. 
For any $i$ that $(i,j)\in I_H,$ we then have 
$$\sum_{j}r_{ij} a_i^{\top} w_j^{\top}=a_i^{\top}(\sum_{j}r_{ij} w_j^{\top})=0.$$ 
Since the row vectors in $\Tilde{A}$ are linearly independent, $\sum_{j}r_{ij} b_j=0$. This shows that the set of all $(W,B)$ that satisfies condition (\romannumeral2) is a complement of Lebesgue null set in $\mathbb{R}^{\#{\rm weights}+\#{\rm biases}}$.
When the parameters $\theta$ are in a complement of Lebesgue measure-zero set in $\mathbb{R}^{\#{\rm weights}+\#{\rm biases}}$, $R_{\mathcal{N},\theta}=K_{\mathcal{N},\theta}$. Moreover, by Lemma~\ref{lemma:W-rank} and Lemma~\ref{lemma:W-rank2}, all the $(W,B)$ that does not satisfy the two conditions in Proposition~\ref{col:max-one-layer} also form a measure-zero set in $\mathbb{R}^{(N+1)\times N^{\prime}}$. This indicates $$E_{\theta\sim \mu}(R_{\mathcal{N},\theta})=\left(\sum_{i=0}^{N}\tbinom{N^{\prime}}{i}\right)^{D^*}.$$
\end{proof}

\subsection{Proof of Theorem \ref{thm:asym}}

\begin{proof}[Proof of Theorem \ref{thm:asym}]
By Theorem \ref{thm:one-layer-eq},  $$R_{\mathcal{N},\theta}=\left(\sum_{i=0}^{N}\tbinom{N^{\prime}}{i}\right)^{D^*}.$$ Since $N^{\prime}$ tends to infinity and $N=\Theta(1)$, $\tbinom{N^{\prime}}{N}=\Theta({N^{\prime}}^N)$ and $\tbinom{N^{\prime}}{i}=\Theta({N^{\prime}}^i)$, which gives 
$$R_{\mathcal{N},\theta}=\Theta({N^{\prime}}^{DN}).$$
\end{proof}

\section{Proof of Linear Regions for Multi-Layer GCNs}\label{sec:3.1}
Let $\mathcal{N}$ be a GCN with $L$ hidden convolutional layers. Suppose that there are $N_0$ input features and $N_l$ output features in the $l$-th layer.
First of all, we give the proof of the upper bound.
\begin{proof}[Proof of Theorem~\ref{thm:bounds}(\romannumeral1)]
 First, when $L=1$, the inequality holds in  \eqref{ineq:low-multi}. Now suppose that $L\geq 2$ and that the inequality is true for $L-1$. Let $\mathcal{N}^*$ be the GCN consist the first $L-1$ layers of $\mathcal{N}$. Then, it can calculated that
 $$
 R_{\mathcal{N}^*}\leq R_{\mathcal{N}^{\prime\prime}}\prod\limits_{l=2}^{L-1}\left(\sum\limits_{i=0}^{DN}\tbinom{DN_l}{i}\right).$$ 
Suppose that the first $L-1$ layers divide the input space into $m$ distinct linear regions $R_i$ ($1\leq i\leq m$). Let $\mathcal{F}_{\mathcal{N}}$ be the function represented by $\mathcal{N}$ and $\mathcal{F}_{\mathcal{N}^*}$ be the function represented by $\mathcal{N}^*$. When restricted into $\R_i$, $\mathcal{F}_{\mathcal{N}^*}$ is a linear function and $\mathcal{F}_{\mathcal{N}}$ is a composition of $\mathcal{F}_{\mathcal{N}^*}$ and the function represented by the last layer which is a composition of linear function and activation function. Therefore, $\mathcal{N}$ could be represented as a one-layer fully connected neural network with input dimension $D\times N_0$ and the output dimension $D\times N_L$. We thus obtain 
$$R_{\mathcal{N}}\leq R_{\mathcal{N}^*}\left(\sum\limits_{i=0}^{DN}\tbinom{DN_L}{i}\right)\leq R_{\mathcal{N}^{\prime\prime}}\prod\limits_{l=2}^{L}\left(\sum\limits_{i=0}^{DN}\tbinom{DN_l}{i}\right).$$
\end{proof}
Before starting the proof of the lower bound, we need several lemmas.
\begin{lemma}\label{lemma:eigen}
An eigenvalue of $\hat{A}$ is 1 and one of the corresponding eigenvector is $M^{\frac{1}{2}}J$, where $J$ is a vector with all elements equal to 1. This implies that all elements in the eigenvector $M^{\frac{1}{2}}J$ are nonnegative.
\end{lemma}
\begin{proof}
Denote $\lambda$ as an eigenvalue and $x$ as the corresponding eigenvector. We then have $\hat{A}x=\lambda x$.
Let $y=M^{\frac{1}{2}}x$. Then, $M^{-1}(A+I)y=\lambda y$, and we could let $\lambda=1$ and $y=J$. Thus, $x=M^{\frac{1}{2}}J$.
\end{proof}

\begin{proof}[Proof of Theorem~\ref{thm:bounds}(\romannumeral2)]
Now we are ready to prove the lower bound.
%
Consider the ReLU GCN model
\begin{equation}
    X^{(l+1)}=h(\hat{A}X^{(l)}W),
\end{equation}
where $\hat{A}=M^{-\frac{1}{2}}(A+I)M^{-\frac{1}{2}}$ and $M={\rm diag}\bigl(\sum\limits_{j=1}^{D}(A_{ij}+1),\: i=1,2,..., D\bigr)$.
We now prove the lower bound by constructing a specific multi-layer GCN that meets the lower bound. As $N_1\geq N$, we can divide $N_l$ into $N$ parts: $n_l^{(k)}$, $k=1,2,..., N$, $\sum\limits_{k=1}^{N}n_l^{(k)}=N_l$ and $n_l^{(k)}\geq 1$. For simplicity, we let $n_l^{(k)}=\left\lfloor \frac{N_l}{N} \right\rfloor$ denoted as $p$. We then let
\begin{equation}
    W_{a,b}^{l}=\left\{
    \begin{array}{ccll}
        p, &  &  a=k,\; b=p(k-1)+1,& k=1,2,..., N,\\[1mm]
        2p, &   &  a=k,\; p(k-1)+2 \leq b \leq pk,& k=1,2,..., N,\\[1mm]
        0, &  &  \hbox{otherwise},& \\
    \end{array}\right.
\end{equation}
and
\begin{equation}
    B_{a,b}^{l}=
        -2(m-1)\sqrt{r_a},\quad  b=p(k-1)+m,\quad  k=1,2,..., N,\quad m=1,2,...,p,
\end{equation}
where $r_a$ is the ($a,a$)-th entry of the matrix $R$.

Denote $Y^{(l)}=\hat{A}X^{(l)}=(Y_1,..., Y_N)$. Then,
\begin{equation}
   Z^l_{a,b}= \left\{
    \begin{array}{cclll}
        pY^{(l)}_{a,b}, &  &  a=k,& b=p(k-1)+1& k=1,2,..., N,\\[1mm]
        2(pY^{(l)}_{a,b}-(m-1)\sqrt{r_a}), &   &   b=p(k-1)+m,&  k=1,2,..., N,& m=1,2,..., p,\\[1mm]
        0, & &  otherwise.&& \\
    \end{array}\right.
\end{equation}
The map $X^{(l+1)}=\max\{0,Z^l_{a,b}\}$ defines a function $X^{(l+1)}=\Phi_{l+1}(X^{(l)})$.

Next, we define an intermediate linear layer without activation function from $X^{(l)}$ to $X^{(l+1)}$ to $U^{(l+1)}$. We define weights by
\begin{equation}
    W_{a,b}^{*l}=\left\{
    \begin{array}{cclll}
        (-1)^{m+1},&  & a=p(k-1)+m& k=1,2,..., N& m=1,2,..., p,\\[1mm]
        0, &  &  \hbox{otherwise} & &\\
    \end{array}\right.
\end{equation}
and $B^{*l}=0$. We then have the equation
\begin{equation}
    U^{(l+1)}_{a,b} = \sum\limits_{j=1}^{p}(-1)^{j+1}X^{(l+1)}_{a,(b-1)p+j},
\end{equation}
which defines an affine function
\begin{equation}
  Y_{l+1}=\Phi^*_1(X_{l+1}).
\end{equation}

For each $i,j\in \mathbb{N^+}$, let
\begin{equation}
    \psi_i^c(x)=\left\{
    \begin{array}{ccl}
        \max\{0,x\},&   i=1,\\[1mm]
        \max\{0,2x-(2i-2c)\} & i\geq 2,\\
    \end{array}\right.
\end{equation}
where $c$ is a constant, and we let
\begin{equation}
    \phi_j^c(x)=\sum\limits_{i=1}^{j}\psi_i^c(jx).
\end{equation}
It can be checked that 
\begin{equation}
    \phi_j^c(x)=\left\{
    \begin{array}{ccl}
        0 &   i=1,\\[1mm]
        jx-i &  \frac{i}{j}c \leq x \leq \frac{i+1}{j}c,\\[1mm]
        i-jx & \frac{i-1}{j}c \leq x \leq \frac{i}{j}c.\\
    \end{array}\right.
\end{equation}
Then, $\phi_j$ is a linear function when $x$ is restricted to the interval $\left[0,\frac{1}{j}c\right],$ $\left[\frac{1}{j}c,\frac{2}{j}c\right],\dots$ $,\left[\frac{j-1}{j}c,c\right]$ and $$\phi_j\left(\left[\frac{i}{j}c,\frac{i+1}{j}c\right]\right)=\left[0,c\right].$$
It can verified that $U^{(l+1)}_{a,b}=\phi_p^{\sqrt{r_a}}(Y^{(l)}_{a,b})$ and this determines the mapping 
$$\Psi_{l+1}=\Phi^*_{l+1}\circ\Phi_{l+1}\circ\mathscr{A},$$ 
where $\mathscr{A}$: $X\rightarrow AX$. 
Since $\mathscr{A}$ sends $\prod\limits_{i=1}^{D}\left[0,\sqrt{r_i}\right]$ into $\prod\limits_{i=1}^{D}\left[0,\sqrt{r_i}\right]$
, the $\Psi_{l+1}$ sends $p^{N\cdot \rank(A)}$ distinct linear regions in the original space
$$
\prod\limits_{i=1}^{D}\left\{\left[0,\frac{1}{p}\sqrt{r_i}\right],\left[\frac{1}{p}\sqrt{r_i},\frac{2}{p}\sqrt{r_i}\right],..., \left[\frac{p-1}{p}\sqrt{r_i},\sqrt{r_i}\right]\right\}^N
$$
in $\prod\limits_{i=1}^{D}\left[0,\sqrt{r_i}\right]^N$ onto the same hypercubes. We can use similar mappings $\Psi_{1}$, $\Psi_{2},\dots$, and $\Psi_{L}$ to send the $p^{N\cdot \rank(A)}$ distinct linear regions in the input space onto the same hypercubes $\prod\limits_{i=1}^{D}\left[0,\sqrt{r_i}\right]^N$. The left part of the network is a one-layer ReLU GCN with $D$ nodes and $N$ input features. 

Also, a one-layer ReLU GCN with $D$ nodes and $N$ features for each node and output dimension $D\times N_L$ can partition the hypercube $\prod\limits_{i=1}^{D}\left[0,\sqrt{r_i}\right]^N$ into $R_{\mathcal{N^{\prime}}}$ regions. Putting the networks from $X^{(0)}$ to $U^{(L-1)}$ and $U^{(L-1)}$ to $X^{(L)}$ together, we obtain the lower bound in \eqref{ineq:low-multi}.
\end{proof}

\section{Experiments}
\paragraph{Details on Figure \ref{fig:linear-regions}} We consider the two-layer and three layer GCN with one input feature and three nodes in the graph. We choose a 2D slice of the input space and show the linear regions in it. The 2D slice is determined by three random points. We choose 90000 points uniformly in $[-10,10]^2$ from the slice and evaluate the gradient of each point. Then we plot the gradient of each point using colors determined by the gradient. Therefore, the region with the same color is the region with the same gradient and could be considered as one linear region in the slice. 

\paragraph{Estimate for Linear Regions of GCNs with 1 and 2 layers} Suppose that the adjacency matrix is 
\begin{equation}\label{eq:adjmat}
    \Tilde{A}=\left (
 \begin{matrix}
   \frac{1}{3} & \frac{1}{\sqrt{6}} & \frac{1}{\sqrt{6}} \\
    \frac{1}{\sqrt{6}} & \frac{1}{2} &  0 \\
    \frac{1}{\sqrt{6}} &  0 & \frac{1}{2}
  \end{matrix}
  \right).
\end{equation}
 
For 1-layer GCN with 1 input feature and $N_1$ output features, from \eqref{eq:expected R_N}, $R_{\mathcal{N}}=\left(\frac{N_1^2}{2}+\frac{N_1}{2}+1\right)^3$ almost surely. The upper bounds in Proposition~\ref{prop:nn-eq} and Lemma~\ref{lemma:trivial-bound} are $R_{\mathcal{N}}\leq \sum_{i=0}^{6}\tbinom{3N_1}{i}$ and $R_{\mathcal{N}}\leq 2^{3N_1}$. For the 2-layer GCN with 1 input feature and $N_1$ output features for the first layer and 3 output features for the second layer, the upper and lower bounds are given by $343\left\lfloor \frac{N_1}{2} \right\rfloor ^3\leq R_{\mathcal{N}}\leq \left(\frac{N_1^2}{2}+\frac{N_1}{2}+1\right)^3\sum_{i=0}^{3N_1}\tbinom{3N_1}{i}$.

\begin{figure}[t]
\centering   
 \begin{minipage}{\textwidth}
 \centering
\begin{minipage}{0.48\textwidth}
	\centering         
	\includegraphics[width=\textwidth]{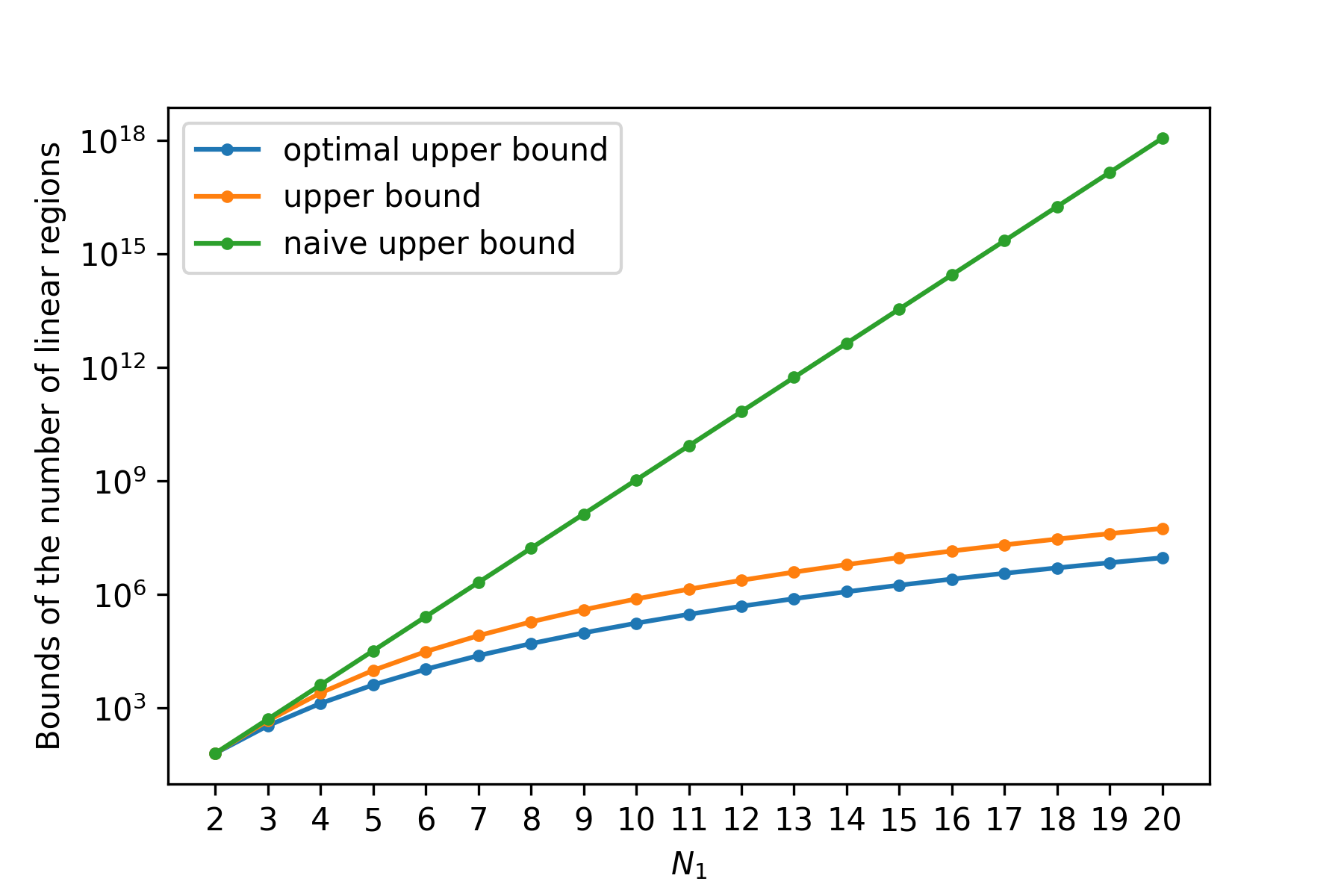}   
	\end{minipage}
\begin{minipage}{0.48\textwidth}
	\centering         
	\includegraphics[width=\textwidth]{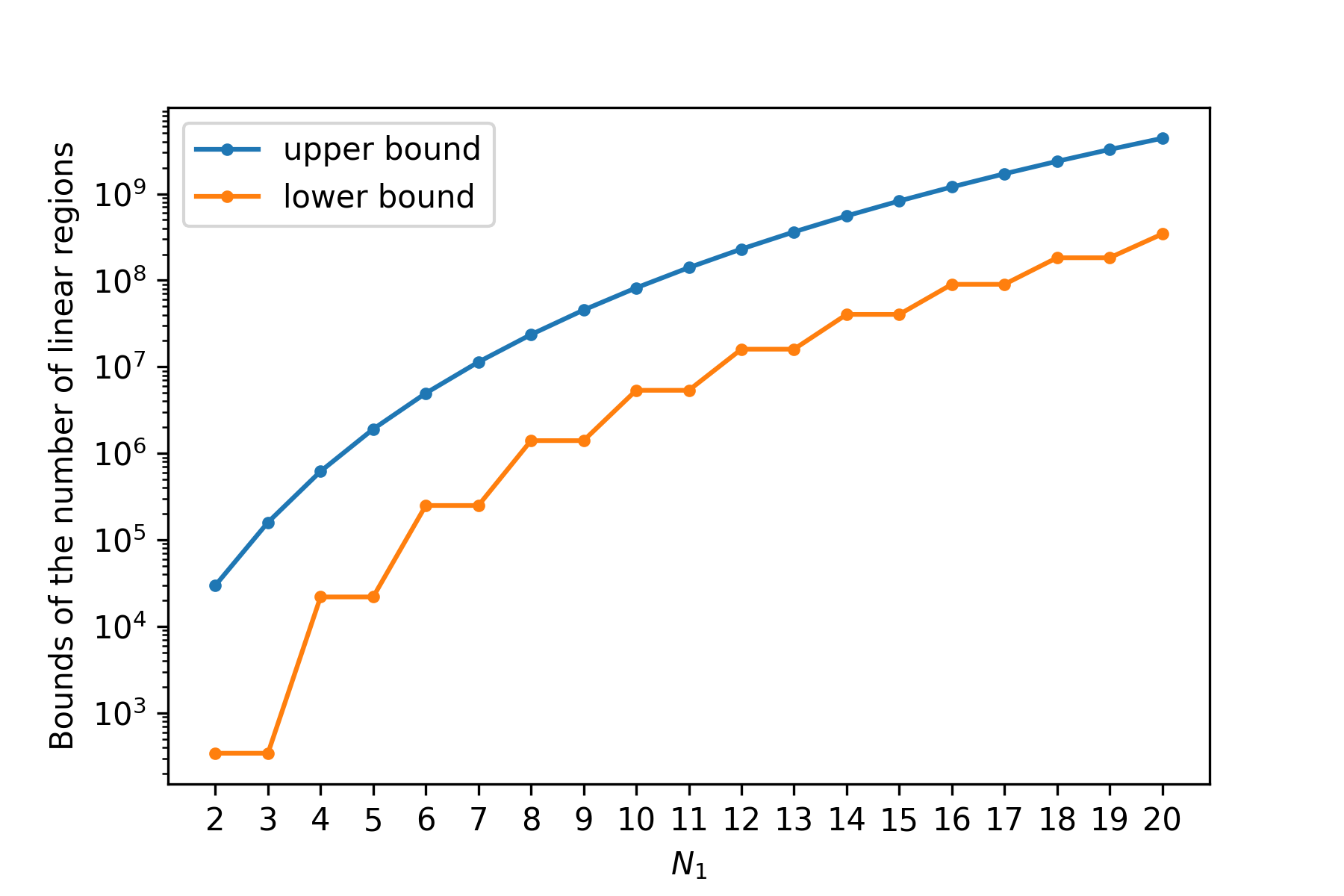}   
	\end{minipage}
\end{minipage}\vspace{1mm}
\caption{Bounds for GCNs with 1 layer (left) and 2 layers (right) on an input graph which consists of 4 nodes and 3 edges. The adjacency matrix is shown in \eqref{eq:adjmat}. The left image shows the optimal upper bound (green) gained from Theorem~\ref{thm:one-layer-eq}, upper bound (orange) from Proposition~\ref{prop:nn-eq} and the naive upper bound (blue) from Lemma~\ref{lemma:trivial-bound} for a 1-layer GCN with 2 input feature and $N_1$ output features where $N_1$ ranges from 2 to 20. The right image shows the lower bound (orange) and upper bound (blue) obtained from Theorem~\ref{thm:bounds} for a 2-layer GCN with 2 input feature and $N_1$ and 3 output features for the first and second layers.}
\label{fig:linear-region}
\end{figure}

\end{document}